\documentclass{article}
\usepackage[utf8]{inputenc}
\usepackage{array}
\usepackage{float}
\usepackage{bm}
\usepackage{multirow}
\usepackage{amsmath}
\usepackage{amsthm}
\usepackage{graphicx}
\usepackage{setspace}

\makeatletter

\providecommand{\tabularnewline}{\\}
\floatstyle{ruled}
\newfloat{algorithm}{tbp}{loa}
\providecommand{\algorithmname}{Algorithm}
\floatname{algorithm}{\protect\algorithmname}

\theoremstyle{plain}
\newtheorem{thm}{\protect\theoremname}

\usepackage{ijcai18}

\usepackage{times}
\usepackage{xcolor}
\usepackage{soul}
\usepackage[small]{caption}

\title{IJCAI--18 Formatting Instructions\thanks{These match the formatting instructions of IJCAI-07. The support of IJCAI, Inc. is acknowledged.}}

\author{
Shuo Chen$^1$,
Chen Gong$^1$,
Jian Yang$^1$,
Xiang Li$^1$,
Yang Wei$^1$,
Jun Li$^{1,2}$\\
$^1$DeepInsight@PCALab, Nanjing University of Science and Technology.\\
$^2$Department of Electrical and Computer Engineering, Northeastern University. \\
 \{shuochen, chen.gong, csjyang, xiang.li.implus, csywei\}@njust.edu.cn, junl.mldl@gmail.com.
}
\usepackage{amsfonts}
\usepackage{epsfig}
\usepackage{times}
\usepackage{epsfig}
\usepackage{graphicx}
\usepackage{color}
\usepackage{hyperref}

\makeatother

\providecommand{\theoremname}{Theorem}

\begin{document}

\title{Adversarial Metric Learning}
\maketitle
\begin{abstract}
In the past decades, intensive efforts have been put to design various
loss functions and metric forms for metric learning problem. These
improvements have shown promising results when the test data is similar
to the training data. However, the trained models often fail to produce
reliable distances on the ambiguous test pairs due to the distribution
bias between training set and test set. To address this problem, the
Adversarial Metric Learning (AML) is proposed in this paper, which
automatically generates adversarial pairs to remedy the distribution
bias and facilitate robust metric learning. Specifically, AML consists
of two adversarial stages, \emph{i.e.} confusion and distinguishment.
In confusion stage, the ambiguous but critical adversarial data pairs
are adaptively generated to mislead the learned metric. In distinguishment
stage, a metric is exhaustively learned to try its best to distinguish
both the adversarial pairs and the original training pairs. Thanks
to the challenges posed by the confusion stage in such competing process,
the AML model is able to grasp plentiful difficult knowledge that
has not been contained by the original training pairs, so the discriminability
of AML can be significantly improved. The entire model is formulated
into optimization framework, of which the global convergence is theoretically
proved. The experimental results on toy data and practical datasets
clearly demonstrate the superiority of AML to the representative state-of-the-art
metric learning methodologies.
\end{abstract}

\section{Introduction}

The calculation of similarity or distance between a pair of data points
plays a fundamental role in many machine learning and pattern recognition
tasks such as retrieval \cite{yang2010_retrival_metric_learning},
verification \cite{IJCAI2017_verification}, and classification \cite{yang2016_empirical}.
Therefore, ``Metric Learning'' \cite{bishop2006pattern,weinberger2009distance}
was proposed to enable an algorithm to wisely acquire the appropriate
distance metric so that the precise similarity between different examples
can be faithfully reflected.
\begin{figure}[t]
\begin{spacing}{0.3}
\noindent \includegraphics[scale=0.395]{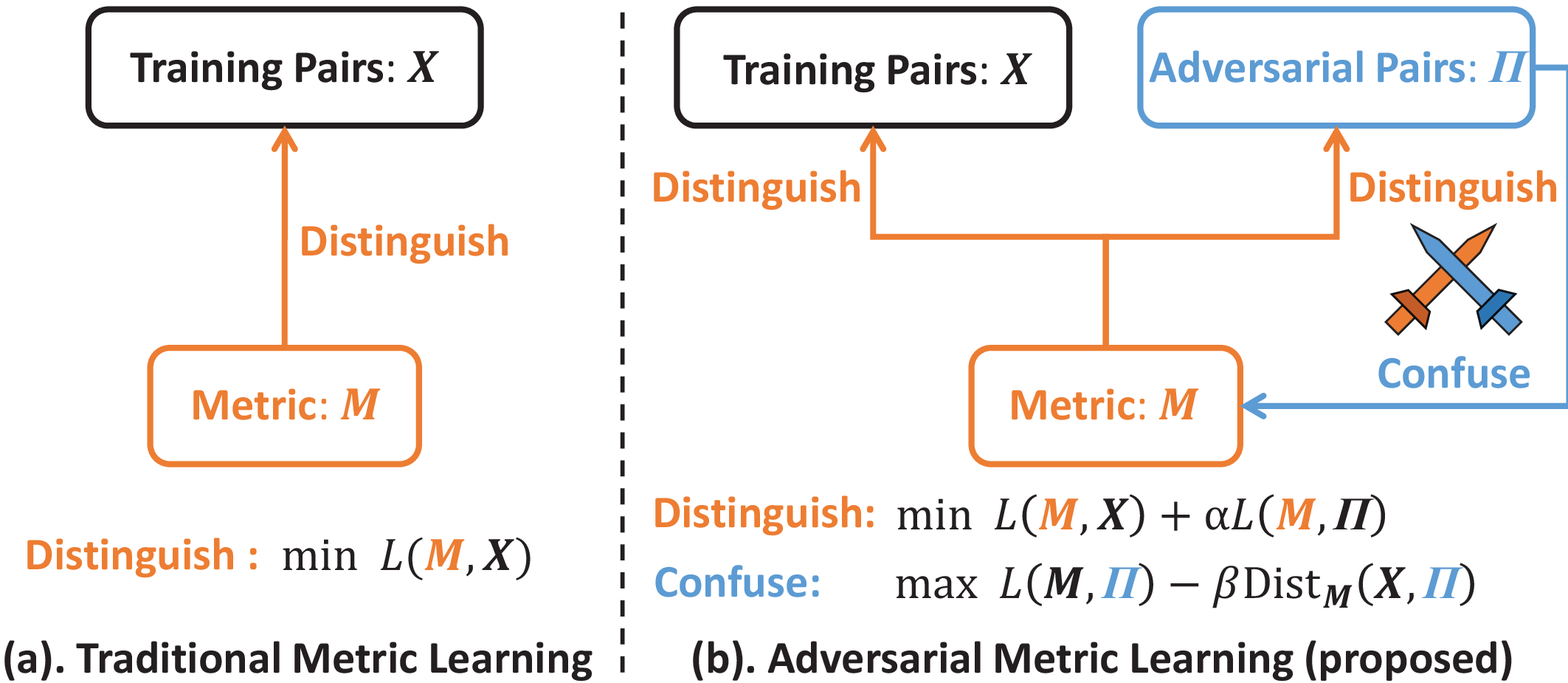}
\end{spacing}

\begin{spacing}{0.5}
\noindent \caption{\label{fig:The-framework-of}The comparison of traditional metric
learning and our proposed model. (a) Traditional metric learning directly
minimizes the loss $L(\boldsymbol{M},\boldsymbol{X})$ to distinguish
the training pairs. (b) Our proposed AML contains a distinguishment
stage and a confusion stage, and the original training pairs and the
adversarial pairs are jointly invoked to obtain an accurate $\boldsymbol{M}$.
The confusion stage learns the adversarial pairs $\boldsymbol{\varPi}$
in the local region of $\boldsymbol{X}$, by maximizing $L(\boldsymbol{M},\boldsymbol{\varPi})$
as well as the distance regularizer $-\text{Dist}_{\boldsymbol{M}}(\boldsymbol{X},\boldsymbol{\varPi})$.}
\end{spacing}
\end{figure}

In metric learning, the similarity between two example vectors $\boldsymbol{x}$
and $\boldsymbol{x}'$ is usually expressed by the distance function
$\text{Dist}(\boldsymbol{x},\thinspace\boldsymbol{x}')$. Perhaps
the most commonly used distance function is Mahalanobis distance,
which has the form $\text{Dist}_{\boldsymbol{M}}(\boldsymbol{x},\thinspace\boldsymbol{x}')=(\boldsymbol{x}-\boldsymbol{x}')^{\top}\boldsymbol{M}(\boldsymbol{x}-\boldsymbol{x}')$\footnote{For simplicity, the notation of ``square'' on $\text{Dist}_{\boldsymbol{M}}(\boldsymbol{x},\thinspace\boldsymbol{x}')$
has been omitted and it will not influence the final output.}. Here the symmetric positive definite (SPD) matrix $\boldsymbol{M}$
should be learned by an algorithm to fit the similarity reflected
by training data. By decomposing $\boldsymbol{M}$ as $\boldsymbol{M}=\boldsymbol{P}^{\top}\boldsymbol{P}$,
we know that Mahalanobis distance intrinsically calculates the Euclidean
distance in a projected linear space rendered by the projection matrix
$\boldsymbol{P}$, namely $\text{Dist}_{\boldsymbol{M}}(\boldsymbol{x},\thinspace\boldsymbol{x}')=||\boldsymbol{P}(\boldsymbol{x}-\boldsymbol{x}')||_{2}^{2}$.
Consequently, a large amount of models were proposed to either directly
pursue the Mahalanobis matrix \cite{davis2007information,zadeh2016geometric,zhang2017efficient_low-rank_metric_learning}
or indirectly learn such a linear projection $\boldsymbol{P}$ \cite{lu2014neighborhood_W_Metric_learning,harandi2017joint_metric_learning}.
Furthermore, considering that above linear transformation is not flexible
enough to characterize the complex data relationship, some recent
works utilized the deep neural networks, \emph{e.g.}\ Convolutional
Neural Network (CNN) \cite{simo2015discriminative,oh2016_lifted_embedding_metric_learning},
to achieve the purpose of non-linearity. Generally, the kernel method
or CNN based nonlinear distance metrics can be summarized as $\text{Dist}_{\boldsymbol{P},\thinspace\mathcal{W}}(\boldsymbol{x},\thinspace\boldsymbol{x}')=||\boldsymbol{P}(\mathcal{W}(\boldsymbol{x})-\mathcal{W}(\boldsymbol{x}'))||_{2}^{2}$,
in which the output of neural network is denoted by the mapping $\mathcal{W}(\cdot)$. 

However, above existing approaches simply learn the linear or non-linear
metrics via designing different loss functions on the original training
pairs. During the test phase, due to the distribution bias of training
set and test set, some ambiguous data pairs that are difficult to
be distinguished by the learned metric may appear, which will significantly
impair the algorithm performance. To this end, we propose the Adversarial
Metric Learning (AML) to learn a robust metric, which follows the
idea of adversarial training \cite{goodfellow2014explaining,li2017_adver_mem},
and is able to generate ambiguous but critical data pairs to enhance
the algorithm robustness. As shown in Fig.~\ref{fig:The-framework-of},
compared with the traditional metric learning methods that only distinguish
the given training pairs, our AML learns the metric to distinguish
both original training pairs and the generated adversarial pairs.
Here, the adversarial data pairs are automatically synthesized by
the algorithm to confuse the learned metric as much as possible. The
adversarial pairs $\boldsymbol{\varPi}$ and the learned metric $\boldsymbol{M}$
form the adversarial relationship and each of them tries to ``beat''
the other one. Specifically, adversarial pairs tend to introduce the
ambiguous examples which are difficult for the learned metric to correctly
decide their (dis)similarities (\emph{i.e.}\ confusion stage), while
the metric makes its effort to discriminate the confusing adversarial
pairs (\emph{i.e.}\ distinguishment stage). In this sense, the adversarial
pairs are helpful for our model to acquire the accurate metric. To
avoid the iterative competing, we convert the adversarial game to
an optimization problem which has the optimal solution from the theoretical
aspects. In the experiments, we show that the robust Mahalanobis metric
learned by AML is superior to the state-of-the-art metric learning
models on popular datasets with classification and verification tasks.

The most prominent advantage of our proposed AML is that the extra
data pairs (\emph{i.e.}\ adversarial pairs) are explored to boost
the discriminability of the learned metric. In fact, several metric
learning models have been proposed based on data augmentations \cite{ahmed2015improved_deep_metric,zagoruyko2015_2-channel_siamase},
or pair perturbations \cite{NIPS2015_RVML,ye_learning_M-distance}.
However, the virtual data generated by these methods are largely based
on the prior which may significantly differ from the practical test
data, so their performances are rather limited. In contrast, the additional
adversarial pairs in AML are consciously designed to mislead the learning
metric, so they are formed in an intentional and realistic way. Specifically,
to narrow the searching space of adversarial pairs, AML establishes
the adversarial pairs within neighborhoods of original training pairs
as shown in Fig.~\ref{fig:The-distributions-of}. Thanks to the learning
on both real and generated pairs, the\emph{ }discriminability of our
method can be substantially improved. 

\begin{spacing}{0.5}
\noindent 
\begin{figure}[t]
\begin{spacing}{0.3}
\noindent \begin{centering}
\includegraphics[scale=0.34]{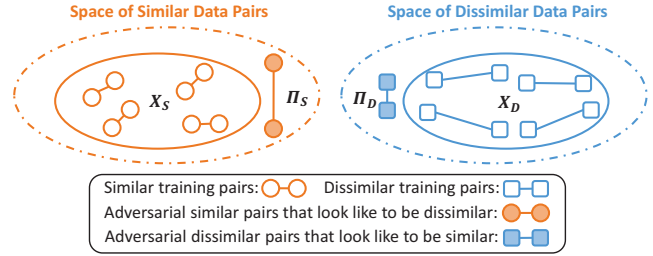}
\par\end{centering}
\end{spacing}
\caption{Generations of adversarial similar pairs and adversarial dissimilar
pairs. Similar and dissimilar pairs are marked with orange balls and
blue blocks, respectively. Hollow balls and blocks denote the original
training examples, while filled balls and blocks denote the adversarial
examples which are automatically generated by our model. Note that
the generated two points constituting the adversarial similar pairs
(\emph{i.e.}\ $\boldsymbol{\varPi}_{S}$) are far from each other,
which describe the extreme cases for two examples to be similar pairs.
Similarly, the generated two points constituting the adversarial dissimilar
pairs (\emph{i.e.}\ $\boldsymbol{\varPi}_{D}$) are closely distributed,
which depict the worst cases for two examples to be dissimilar pairs.\label{fig:The-distributions-of}}
\end{figure}

\end{spacing}

The main contributions of this paper are summarized as:
\begin{itemize}
\item We propose a novel framework dubbed Adversarial Metric Learning (AML),
which is able to generate adversarial data pairs in addition to the
original given training data to enhance the model discriminability.
\item AML is converted to an optimization framework, of which the convergence
is analyzed.
\item AML is empirically validated to outperform state-of-the-art metric
learning models on typical applications. 
\end{itemize}

\section{Adversarial Metric Learning\label{sec:Methodology}}

We first introduce some necessary notations in Section \ref{subsec:Preliminaries},
and then explain the optimization model of the proposed AML in Section
\ref{subsec:Model-Establishment}. Finally, we provide the iterative
solution as well as the convergence proof in Section \ref{subsec:Algorithm}
and Section \ref{subsec:Convergence-Analysis}, respectively. 

\subsection{Preliminaries\label{subsec:Preliminaries}}

Let $\boldsymbol{X}=[\boldsymbol{X}_{1},\boldsymbol{X}_{2},\cdots,\boldsymbol{X}_{N}]\in\mathbb{R}^{2d\times N}$
be the matrix of $N$ training example pairs, where $\boldsymbol{X}_{i}=[\boldsymbol{x}_{i}^{\top},\thinspace\boldsymbol{x}_{i}'^{\top}]^{\top}\in\mathbb{R}^{2d}$
$\left(i=1,2,\cdots,N\right)$ consists of a pair of $d$-dimensional
training examples. Besides, we define a label vector $\bm{y}\in\{-1,1\}^{N}$
of which the $i$-th element $y_{i}$ represents the relationship
of the pairwise examples recorded by the $i$-th column of $\boldsymbol{X}$,
namely $y_{i}=1$ if $\boldsymbol{x}_{i}$ and $\boldsymbol{x}_{i}'$
are similar, and $y_{i}=-1$ otherwise. Based on the supervised training
data, the Mahalanobis metric $\boldsymbol{M}\in\mathbb{R}^{d\times d}$
is learned by minimizing the general loss function $L(\boldsymbol{M},\boldsymbol{X},\boldsymbol{y})$,
namely{\setlength\abovedisplayskip{5pt}\setlength\belowdisplayskip{1pt}
\begin{equation}
\underset{\boldsymbol{M}\in\mathcal{S}}{\min}\thinspace\thinspace d(\boldsymbol{M})=L(\boldsymbol{M},\boldsymbol{X},\boldsymbol{y}),\label{eq:traditional_loss}
\end{equation}
}in which $\mathcal{S}$ denotes the feasible set for $\boldsymbol{M}$,
such as SPD constraint \cite{arsigny2007geometric_SPD}, bounded constraint
\cite{yang2016_empirical}, low-rank constraint \cite{harandi2017joint_metric_learning},
\emph{etc}. In our proposed AML, $N'$ generated adversarial pairs
are denoted by the matrix $\boldsymbol{\varPi}=[\boldsymbol{\varPi}_{1},\boldsymbol{\varPi}_{2},\cdots,\boldsymbol{\varPi}_{N'}]\in\mathbb{R}^{2d\times N'}$,
where $\boldsymbol{\varPi}_{i}=[\boldsymbol{\pi}_{i}^{\top},\thinspace\boldsymbol{\pi}_{i}'^{\top}]^{\top}\in\mathbb{R}^{2d}$
$\left(i=1,2,\cdots,N'\right)$ represents the $i$-th generated example
pair. By setting $N'$ to $N$ in this work, the distance $\text{Dist}_{\boldsymbol{M}}(\boldsymbol{X},\boldsymbol{\varPi})$
between $\boldsymbol{X}$ and $\boldsymbol{\varPi}$ is thus defined
as the sum of the Mahalanobis distances between all the pairwise examples
of $\boldsymbol{X}$ and $\boldsymbol{\varPi}$, \emph{i.e.}\ $\text{Dist}_{\boldsymbol{M}}(\boldsymbol{X},\boldsymbol{\varPi})=\sum_{i=1}^{N}\text{Dist}_{\boldsymbol{M}}(\boldsymbol{x}_{i},\boldsymbol{\pi}_{i})+\text{Dist}_{\boldsymbol{M}}(\boldsymbol{x}_{i}',\boldsymbol{\pi}_{i}')$.

\subsection{Model Establishment\label{subsec:Model-Establishment}}

As mentioned in the Introduction, our AML algorithm alternates between
learning the reliable distance metric (\emph{i.e.}\ distinguishment
stage) and generating the misleading adversarial data pairs (\emph{i.e.}\
confusion stage), in which the latter is the core of AML to boost
the learning performance. The main target of confusion stage is to
produce the adversarial pairs $\boldsymbol{\varPi}$ to confuse the
learned metric $\boldsymbol{M}$. That is to say, we should explore
the pair $\boldsymbol{\varPi}_{i}$ of which the similarity predicted
by $\boldsymbol{M}$ is opposite to its true label. Fig.~\ref{fig:The-distributions-of}
intuitively plots the generations of $\boldsymbol{\varPi}$. To achieve
this effect, we search the data pairs $\boldsymbol{\varPi}_{i}$ in
the neighborhood of $\boldsymbol{X}_{i}$ to violate the results predicted
by the learned metric. Specifically, the loss function $\boldsymbol{L}(\boldsymbol{M},\boldsymbol{\varPi},\boldsymbol{y})$
is expected to be as large as possible, while the distance $\text{Dist}_{\boldsymbol{M}}(\boldsymbol{X},\boldsymbol{\varPi})$
is preferred to be a small value in the following optimization objective{\setlength\abovedisplayskip{5pt}\setlength\belowdisplayskip{1pt}
\begin{equation}
\underset{\boldsymbol{\varPi}}{\max}\thinspace\thinspace c(\boldsymbol{\varPi})=L(\boldsymbol{M},\boldsymbol{\varPi},\boldsymbol{y})-\beta\text{Dist}_{\boldsymbol{M}}(\boldsymbol{X},\boldsymbol{\varPi}),\label{eq:adversarial_max}
\end{equation}
}in which the regularizer coefficient $\beta\in\mathbb{R}^{+}$ is
manually tuned to control the size of searching space. Since $\boldsymbol{\varPi}_{i}$
is found in the neighborhood of $\boldsymbol{X}_{i}$, the true label
of $\boldsymbol{\varPi}_{i}$ is reasonably assumed as $y_{i}$, \emph{i.e.}\
the label of $\boldsymbol{X}_{i}$. It means that Eq.~\eqref{eq:adversarial_max}
tries to find data pairs $\boldsymbol{\varPi}_{1},\boldsymbol{\varPi}_{2},\cdots,\boldsymbol{\varPi}_{N}$,
of which their metric results are opposite to their corresponding
true labels $y_{1},y_{2},\cdots,y_{N}$. Therefore, such an optimization
exploits the adversarial pairs $\boldsymbol{\varPi}$ to confuse the
metric $\boldsymbol{M}$. 

Nevertheless, Eq.~\eqref{eq:adversarial_max} cannot be directly taken
as a valid optimization problem, as it is not bounded which means
that Eq.~\eqref{eq:adversarial_max} might not have the optimal solution.
To avoid this problem and achieve the same effect with Eq.~\eqref{eq:adversarial_max},
we convert the maximization of the loss \emph{w.r.t.}\ the true labels
$\boldsymbol{y}$ to the minimization of the loss \emph{w.r.t.}\
the opposite labels $-\boldsymbol{y}$, because the opposite labels
yield the opposite similarities when they are used to supervise the
minimization of the loss function. Then the \emph{confusion} stage
is reformulated as {\setlength\abovedisplayskip{5pt}\setlength\belowdisplayskip{1pt}
\begin{equation}
\underset{\boldsymbol{\varPi}}{\min}\thinspace\thinspace C_{\boldsymbol{M}}(\boldsymbol{\varPi})=L(\boldsymbol{M},\boldsymbol{\varPi},-\boldsymbol{y})+\beta\text{Dist}_{\boldsymbol{M}}(\boldsymbol{X},\boldsymbol{\varPi}).\label{eq:adversarial_min}
\end{equation}
}The optimal solution to the above problem always exists, because
any loss function $L(\cdot)$ and distance operator $\text{Dist}_{\boldsymbol{M}}(\cdot)$
have the minimal values. 

By solving Eq.~\eqref{eq:adversarial_min}, we obtain the generated
adversarial pairs recorded in the matrix $\boldsymbol{\varPi}$, which
can be employed to learn a proper metric. Since these confusing adversarial
pairs are incorrectly predicted, the metric $\boldsymbol{M}$ should
exhaustively distinguish them to improve the discriminability. By
combining the adversarial pairs in $\boldsymbol{\varPi}$ and the
originally available training pairs in $\boldsymbol{X}$, the augmented
training loss utilized in the \emph{distinguishment} stage has a form
of {\setlength\abovedisplayskip{5pt}\setlength\belowdisplayskip{1pt}
\begin{equation}
\underset{\boldsymbol{M}\in\mathcal{S}}{\min}\thinspace\thinspace D_{\boldsymbol{\varPi}}(\boldsymbol{M})=L(\boldsymbol{M},\boldsymbol{X},\boldsymbol{y})+\alpha L(\boldsymbol{M},\boldsymbol{\varPi},\boldsymbol{y}),\label{eq:learning_loss}
\end{equation}
}where the regularizer coefficient $\alpha\in\mathbb{R}^{+}$ is
manually tuned to control the weights of the adversarial data. Furthermore,
to improve both the distinguishment (\emph{i.e.}\ Eq.~\eqref{eq:learning_loss})
and the confusion (\emph{i.e.}\ Eq.~\eqref{eq:adversarial_min}) during
their formed adversarial game, they have to be optimized alternatively,
\emph{i.e.}\ {\setlength\abovedisplayskip{2pt} \setlength\belowdisplayskip{2pt}
\begin{equation}
\begin{cases}
\boldsymbol{M}^{(t+1)} & =\thinspace\arg\min\thinspace D_{\boldsymbol{\varPi}^{(t)}}(\boldsymbol{M}),\thinspace\small\text{(\emph{Distinguishment})},\\
\boldsymbol{\varPi}^{(t+1)} & =\thinspace\arg\min\thinspace C_{\boldsymbol{M}^{(t+1)}}(\boldsymbol{\varPi}),\thinspace\thinspace\thinspace\thinspace\thinspace\thinspace\thinspace\thinspace\thinspace\thinspace\small\text{(\emph{Confusion})}.
\end{cases}\label{eq:AML_iterative}
\end{equation}
}The straightforward implementation of Eq.~\eqref{eq:AML_iterative}
yet confronts two problems in the practical use. Firstly, Eq.~\eqref{eq:AML_iterative}
is iteratively performed, which greatly decreases the efficiency of
the model. Secondly, the iterations with two different functions are
not necessarily convergent \cite{ben2009robust_optization}.

To achieve the similar effect of the direct alternation in Eq.~\eqref{eq:AML_iterative}
while avoiding the two disadvantages mentioned above, the iterative
expression for $\boldsymbol{\varPi}$ is integrated to the optimization
of $\boldsymbol{M}$. Therefore, our AML is ultimately expressed as
a bi-level optimization problem \cite{bard2013practical_bilevel},
namely {\setlength\abovedisplayskip{5pt}\setlength\belowdisplayskip{1pt}
\begin{equation}
\underset{\boldsymbol{M}\in\mathcal{S},\thinspace\boldsymbol{\varPi}^{*}}{\min}\thinspace D_{\boldsymbol{\varPi}^{*}}(\boldsymbol{M})\thinspace\thinspace\thinspace\thinspace\thinspace\thinspace\thinspace\thinspace\thinspace\text{s.t.}\thinspace\thinspace\boldsymbol{\varPi}^{*}=\underset{\boldsymbol{\varPi}}{\arg\min}\thinspace\thinspace C_{\boldsymbol{M}}(\boldsymbol{\varPi}),\label{eq:AML}
\end{equation}
}in which $\boldsymbol{\varPi}^{*}$ denotes the optimal adversarial
pairs matrix, and $C_{\boldsymbol{M}}(\boldsymbol{\varPi})$ is required
to be strictly quasi-convex\footnote{If a function $h(\cdot)$ satisfies $h(\lambda\boldsymbol{x}_{1}+(1-\lambda)\boldsymbol{x}_{2})<\text{max}(h(\boldsymbol{x}_{1}),\thinspace h(\boldsymbol{x}_{2}))$
for all $\boldsymbol{x}_{1}\neq\boldsymbol{x}_{2}$ and $\lambda\in(0,1)$,
then $h(\cdot)$ is strictly quasi-convex. }. Note that the strictly quasi-convex property ensures the uniqueness
of $\boldsymbol{\varPi}^{*}$, and helps to make the problem well-defined.

\subsection{Optimization \label{subsec:Algorithm}}

To implement Eq.~\eqref{eq:AML}, we instantiate the loss $L(\cdot)$
in $D(\cdot)$ and $C(\cdot)$ to obtain a specific learning model.
To make $C_{\boldsymbol{M}}(\boldsymbol{\varPi})$ to be convex, here
we employ the geometric-mean loss \cite{zadeh2016geometric} which
has an unconstrained form of{\setlength\abovedisplayskip{3pt}\setlength\belowdisplayskip{3pt}
\begin{align}
 & L_{\text{g}}(\boldsymbol{M},\boldsymbol{X},\boldsymbol{y})\nonumber \\
 & =\sum_{y_{i}=1}\nolimits\text{Dist}_{\boldsymbol{M}}(\boldsymbol{x}_{i},\boldsymbol{x}_{i}')+\sum_{y_{i}=-1}\nolimits\text{Dist}_{\boldsymbol{M}}'(\boldsymbol{x}_{i},\boldsymbol{x}_{i}'),\label{eq:geometric-loss}
\end{align}
}where the loss of dissimilar data pairs is expressed as $\text{Dist}_{\boldsymbol{M}}'(\boldsymbol{x}_{i},\boldsymbol{x}_{i}')=\text{Dist}_{\boldsymbol{M}^{-1}}(\boldsymbol{x}_{i},\boldsymbol{x}_{i}')$
to increase the distances between dissimilar examples. Moreover, we
substitute the loss $L(\cdot)$ in Eq.~\eqref{eq:AML} with $L_{\text{g}}(\cdot)$,
and impose the SPD constraint on $\boldsymbol{M}$ for simplicity,
namely $\mathcal{S}=\{\boldsymbol{M}|\boldsymbol{M}\succ0,\boldsymbol{M}\in\mathbb{R}^{d\times d}\}$.
Then the detailed optimization algorithm for Eq.~\eqref{eq:AML} is
provided as follows. 

\textbf{Solving} $\boldsymbol{\varPi}$: We can directly obtain the
closed-form solution (\emph{i.e.}\ the optimal adversarial pairs
$\boldsymbol{\varPi}^{*}$) for optimizing $\boldsymbol{\varPi}$.
Specifically, by using the convexity of $C_{\boldsymbol{M}}(\boldsymbol{\varPi})$,
we let $\nabla C_{\boldsymbol{M}}(\boldsymbol{\varPi})=0$, and arrive
at{\setlength\abovedisplayskip{3pt}\setlength\belowdisplayskip{3pt}
\begin{equation}
\begin{cases}
\frac{1}{2}\nabla_{\boldsymbol{\pi}_{i}}C\!=\!\boldsymbol{M}^{-y_{i}}(\boldsymbol{\boldsymbol{\pi}}_{i}\!-\!\boldsymbol{\boldsymbol{\pi}}_{i}')+\beta\boldsymbol{M}(\boldsymbol{\boldsymbol{\pi}}_{i}\!-\!\boldsymbol{x}_{i})=0,\\
\frac{1}{2}\nabla_{\boldsymbol{\pi}_{i}'}C\!=\!\boldsymbol{M}^{-y_{i}}(\boldsymbol{\boldsymbol{\pi}}_{i}'\!-\!\boldsymbol{\boldsymbol{\pi}}_{i})+\beta\boldsymbol{M}(\boldsymbol{\boldsymbol{\pi}}_{i}'\!-\!\boldsymbol{x}_{i}')=0,
\end{cases}\label{eq:equation_group}
\end{equation}
}which holds for any $i=1,2,\cdots,N$. It is clear that the equation
system Eq.~\eqref{eq:equation_group} has the unique solution{\setlength\abovedisplayskip{3pt}\setlength\belowdisplayskip{3pt}
\begin{equation}
\begin{cases}
\boldsymbol{\boldsymbol{\pi}}_{i}^{*}=(2\boldsymbol{M}^{-y_{i}}+\beta\boldsymbol{M})^{-1}(\boldsymbol{M}^{-y_{i}}\overline{\boldsymbol{x}}_{i}+\beta\boldsymbol{M}\boldsymbol{x}_{i}),\\
\boldsymbol{\boldsymbol{\pi}}_{i}'^{*}\!=(2\boldsymbol{M}^{-y_{i}}+\beta\boldsymbol{M})^{-1}(\boldsymbol{M}^{-y_{i}}\overline{\boldsymbol{x}}_{i}+\beta\boldsymbol{M}\boldsymbol{x}_{i}'),
\end{cases}\label{eq:equation_system}
\end{equation}
}where $\overline{\boldsymbol{x}}_{i}=\boldsymbol{x}_{i}+\boldsymbol{x}_{i}'$.
Such a closed-form solution $\boldsymbol{\varPi}^{*}$ means that
the minimizer of $C_{\boldsymbol{M}}(\boldsymbol{\varPi})$ can be
directly expressed as {\setlength\abovedisplayskip{3pt}\setlength\belowdisplayskip{3pt}
\begin{equation}
\boldsymbol{\varPi}^{*}=\mathcal{F}(\boldsymbol{M}),
\end{equation}
}where $\mathcal{F}(\cdot)$ is a mapping from $\mathbb{R}^{d\times d}$
to $\mathbb{R}^{2d\times N}$ decided by Eq.~\eqref{eq:equation_system}.
Hence Eq.~\eqref{eq:AML} is equivalently converted to{\setlength\abovedisplayskip{5pt}\setlength\belowdisplayskip{1pt}
\begin{equation}
\underset{\boldsymbol{M}\succ0}{\min}\thinspace\thinspace D(\boldsymbol{M})=L_{\text{g}}(\boldsymbol{M},\boldsymbol{X},\boldsymbol{y})+\alpha L_{\text{g}}(\boldsymbol{M},\mathcal{F}(\boldsymbol{M}),\boldsymbol{y}),\label{eq:optimization_model}
\end{equation}
}which is an unconstrained optimization problem regarding the single
variable $\boldsymbol{M}$. 

\textbf{Solving} $\boldsymbol{M}$: The key point is to calculate
the gradient of the second term $L_{\text{g}}(\boldsymbol{M},\mathcal{F}(\boldsymbol{M}),\boldsymbol{y})$.
We substitute $\mathcal{F}(\boldsymbol{M})$ with Eq.~\eqref{eq:equation_system}
and obtain that

\begin{singlespace}
\noindent 
\begin{align}
 & L_{\text{g}}\left(\boldsymbol{M},\mathcal{F}(\boldsymbol{M}),\boldsymbol{y}\right)\nonumber \\
 & =\beta^{2}\sum_{i=1}^{N}\nolimits\widehat{\boldsymbol{x}}_{i}^{\top}\boldsymbol{U}\left(\frac{\boldsymbol{\varLambda}^{3+2y_{i}}}{(2\text{\textbf{I}}+\beta^{2}\boldsymbol{\varLambda}^{1+y_{i}})^{2}}\right)\boldsymbol{U}^{\top}\widehat{\boldsymbol{x}}_{i},\label{eq:term_2}
\end{align}
where $\widehat{\boldsymbol{x}}_{i}=\boldsymbol{x}_{i}-\boldsymbol{x}_{i}'$,
and $\boldsymbol{U}\boldsymbol{\varLambda}\boldsymbol{U}^{\top}$
is the eigen-decomposition of $\boldsymbol{M}$. Each term to be summed
in the above Eq.~\eqref{eq:term_2} can be compactly written as $H_{i}(\boldsymbol{M})=\widehat{\boldsymbol{x}}_{i}^{\top}\boldsymbol{U}\boldsymbol{h}_{i}(\boldsymbol{\varLambda})\boldsymbol{U}^{\top}\widehat{\boldsymbol{x}}_{i}$,
where $\boldsymbol{h}_{i}(\boldsymbol{\varLambda})=\text{diag}\left(h_{i}(\lambda_{1}),h_{i}(\lambda_{2}),\cdots,h_{i}(\lambda_{d})\right)$
and $h_{i}$ is a differentiable function. The gradient of such a
term can be obtained from the properties of eigenvalues and eigenvectors
\cite{bellman1997introduction_matrix}, namely $\partial\lambda_{i}=\boldsymbol{U}_{i}^{\top}\partial\boldsymbol{M}\boldsymbol{U}_{i},$
and $\partial\boldsymbol{U}_{i}=(\lambda_{i}\text{\textbf{I}}-\boldsymbol{M})^{\dagger}\partial\boldsymbol{M}\boldsymbol{U}_{i}.$
By further leveraging the chain rule of function derivatives \cite{petersen2008matrix_cookbook},
the gradient of Eq.~\eqref{eq:term_2} can be expressed as{\setlength\abovedisplayskip{1pt} \setlength\belowdisplayskip{1pt}
\begin{equation}
\nabla_{\boldsymbol{M}}L_{\text{g}}\left(\boldsymbol{M},\mathcal{F}(\boldsymbol{M}),\boldsymbol{y}\right)=\beta^{2}\sum_{i=1}^{N}\nolimits\nabla H_{i},\label{eq:gradient_2}
\end{equation}
}in which {\setlength\abovedisplayskip{3pt} \setlength\belowdisplayskip{3pt}
\begin{align}
\nabla H & _{i}=\sum_{j=1}^{d}\nolimits h_{i}'(\lambda_{j})\left(\widehat{\boldsymbol{x}}_{i}^{\top}\boldsymbol{U}_{j}\boldsymbol{U}_{j}^{\top}\widehat{\boldsymbol{x}}_{i}\right)\boldsymbol{U}_{j}\boldsymbol{U}_{j}^{\top}\nonumber \\
 & \thinspace\thinspace+\sum_{j=1}^{d}\nolimits2h_{i}(\lambda_{j})(\lambda_{j}\text{\textbf{I}}-\boldsymbol{M})^{\dagger}\widehat{\boldsymbol{x}}_{i}\widehat{\boldsymbol{x}}_{i}^{\top}\boldsymbol{U}_{j}\boldsymbol{U}_{j}^{\top}.\label{eq:gradient_H_i}
\end{align}
}Finally, the gradient of $D(\boldsymbol{M})$ in Eq.~\eqref{eq:optimization_model}
equals to{\setlength\abovedisplayskip{3pt} \setlength\belowdisplayskip{3pt}
\begin{align}
 & \nabla D(\boldsymbol{M})=\boldsymbol{A}-\boldsymbol{M}^{-1}\boldsymbol{B}\boldsymbol{M}^{-1}+\alpha\beta^{2}\sum_{i=1}^{N}\nolimits\nabla H_{i},\label{eq:gradient_F}
\end{align}
}in which the matrices $\boldsymbol{A}=\sum_{y_{i}=1}(\boldsymbol{x}_{i}-\boldsymbol{x}_{i}')(\boldsymbol{x}_{i}-\boldsymbol{x}_{i}')^{\top}$
and $\boldsymbol{B}=\sum_{y_{i}=-1}(\boldsymbol{x}_{i}-\boldsymbol{x}_{i}')(\boldsymbol{x}_{i}-\boldsymbol{x}_{i}')^{\top}$.
It should be noticed that $(\lambda_{i}\text{\textbf{I}}-\boldsymbol{M})^{\dagger}$
can be calculated efficiently for the SPD matrix $\boldsymbol{M}$,
which only depends on the eigen-decomposition.
\end{singlespace}

Now we can simply employ the gradient-based method for SPD optimization
to solve our proposed model. By following the popular SPD algorithm
in the existing metric learning models \cite{ye_learning_M-distance,LuoLei_AAAI},
the projection operator is utilized to remain the SPD effect. Specifically,
for a symmetric matrix $\boldsymbol{M}=\boldsymbol{U}\boldsymbol{\varLambda}\boldsymbol{U}{}^{\top}$,
the projection $\mathcal{P}_{\mathcal{S}}(\cdot)$ from $\mathbb{R}^{d\times d}$
to $\mathbb{R}^{d\times d}$ truncates negative eigenvalues of $\boldsymbol{\varLambda}$,
\emph{i.e.}\ {\setlength\abovedisplayskip{3pt} \setlength\belowdisplayskip{3pt}
\begin{equation}
\mathcal{P}_{\mathcal{S}}(\boldsymbol{M})=\boldsymbol{U}\text{max}(\boldsymbol{\varLambda},0)\boldsymbol{U}{}^{\top}.
\end{equation}
}It can be proved that the metric $\boldsymbol{M}$ remains symmetry
after the gradient descent, so the projection operator is leveraged
in the gradient descent to find the optimal solution. The pseudo-code
for solving Eq.~\eqref{eq:AML} is provided in Algorithm \ref{alg:Solving-AML-Eq.},
where the step size $\rho$ is recommended to be fixed to $0.001$
in our experiments. 
\begin{algorithm}
\caption{Solving AML in Eq.~\eqref{eq:AML} via gradient descent.\label{alg:Solving-AML-Eq.}}

\textbf{Input: }Training data pairs encoded in $\boldsymbol{X}$;
labels $\boldsymbol{y}$; parameters $\alpha$, $\beta$, $\rho$.

\textbf{Initialize: $t=1$}; $\boldsymbol{M}^{(t)}=\text{\textbf{I}}$.\textbf{ }

\textbf{Repeat:}
\begin{enumerate}
\item Compute the gradient $\nabla D(\boldsymbol{M}^{(t)})$ by Eq.~\eqref{eq:gradient_F};
\item Update $\boldsymbol{M}^{(t+1)}=\mathcal{P}_{\mathcal{S}}(\boldsymbol{M}^{(t)}-\rho\nabla D(\boldsymbol{M}^{(t)}))$;
\item Update $t=t+1$;
\end{enumerate}
\textbf{Until Convergence.}

\textbf{Output:} The converged $\boldsymbol{M}$.
\end{algorithm}

\subsection{Convergence Analysis\label{subsec:Convergence-Analysis}}

Since our proposed bi-level optimization problem greatly differs from
the traditional metric learning models, here we provide the theoretical
analysis for the algorithm convergence.

Firstly, to ensure the definition of AML in Eq.~\eqref{eq:AML} is
valid, we prove that the optimal solution $\boldsymbol{\varPi}^{*}$
always exists uniquely by showing the strict (quasi-)convexity of
$C_{\boldsymbol{M}}(\boldsymbol{\varPi})$ when $L(\boldsymbol{M},\boldsymbol{\varPi},-\boldsymbol{y})$
is instantiated by $L_{\text{g}}(\boldsymbol{M},\boldsymbol{\varPi},-\boldsymbol{y})$,
namely
\begin{thm}
\label{proposition:Convexity Proof}Assume that $L(\boldsymbol{M},\boldsymbol{\varPi},-\boldsymbol{y})=L_{\text{g}}(\boldsymbol{M},\boldsymbol{\varPi},-\boldsymbol{y})$.
Then $C_{\boldsymbol{M}}(\boldsymbol{\varPi})$ is strictly convex.
\end{thm}
\begin{proof}
Assume that $\boldsymbol{\varOmega}=[\boldsymbol{\varOmega}_{1},\boldsymbol{\varOmega}_{2},\cdots,\boldsymbol{\varOmega}_{N}]\in\mathbb{R}^{2d\times N}$,
$\boldsymbol{\varOmega}_{i}=[\boldsymbol{\omega}_{i}^{\top},\boldsymbol{\omega}_{i}'^{\top}]^{\top}\in\mathbb{R}^{2d}$
$\left(i=1,2,\cdots,N\right)$ and $\mu\in(0,1)$. By invoking the
SPD property of both $\boldsymbol{M}$ and $\boldsymbol{M}^{-1}$,
we have
\begin{align}
 & L\left(\boldsymbol{M},\mu\boldsymbol{\varPi}+(1-\mu)\boldsymbol{\varOmega},-\boldsymbol{y}\right)\nonumber \\
 & =\sum_{i=1}^{N}\nolimits\left(\mu\widehat{\boldsymbol{\pi}}_{i}+(1-\mu)\widehat{\boldsymbol{\omega}}_{i}\right){}^{\top}\boldsymbol{M}^{-y_{i}}\left(\mu\widehat{\boldsymbol{\pi}}_{i}+(1-\mu)\widehat{\boldsymbol{\omega}}_{i}\right)\nonumber \\
 & <\sum_{i=1}^{N}\nolimits\left(\mu\widehat{\boldsymbol{\pi}}_{i}{}^{\top}\boldsymbol{M}^{-y_{i}}\widehat{\boldsymbol{\pi}}_{i}+(1-\mu)\widehat{\boldsymbol{\omega}}_{i}{}^{\top}\boldsymbol{M}^{-y_{i}}\widehat{\boldsymbol{\omega}}_{i}\right)\nonumber \\
 & =\mu L(\boldsymbol{M},\boldsymbol{\varPi},-\boldsymbol{y})+(1-\mu)L(\boldsymbol{M},\boldsymbol{\varOmega},-\boldsymbol{y}),
\end{align}
where $\widehat{\boldsymbol{\pi}}_{i}=\boldsymbol{\pi}_{i}-\boldsymbol{\pi}_{i}'$,
$\widehat{\boldsymbol{\omega}}_{i}=\boldsymbol{\omega}_{i}-\boldsymbol{\omega}_{i}'$,
and $\boldsymbol{M}^{-y_{i}}$ denotes $\boldsymbol{M}$ $\left(y_{i}=-1\right)$
or $\boldsymbol{M}^{-1}$ $\left(y_{i}=1\right)$. Hence $L(\boldsymbol{M},\boldsymbol{\varPi},\boldsymbol{1}-\boldsymbol{y})$
satisfies the definition of strictly convex function. Similarly, it
is easy to check the convexity of $\text{Dist}_{\boldsymbol{M}}(\boldsymbol{X},\boldsymbol{\varPi})$
\emph{w.r.t.}\ $\boldsymbol{\varPi}$, which completes the proof.
\end{proof}
Furthermore, as we employ the projection $\mathcal{P}_{\mathcal{S}}(\cdot)$
in gradient descent, it is necessary to demonstrate that any result
$\boldsymbol{M}-\rho\nabla D(\boldsymbol{M})$ has the orthogonal
eigen-decomposition. Otherwise, $\mathcal{P}_{\mathcal{S}}(\cdot)$
cannot be executed and the SPD property of $\boldsymbol{M}$ is not
guaranteed. Therefore, in the following Theorem \ref{Proposition: symmetric},
we prove that the gradient matrix $\nabla D(\boldsymbol{M})$ is symmetric,
and thus any converged iteration points are always included in the
feasible set $\mathcal{S}$.
\begin{thm}
\label{Proposition: symmetric} For any differentiable functions $h_{i}(\lambda)$
$(i=1,2,\cdots,N)$, the matrix $\nabla D(\boldsymbol{M})$ in Eq.~\eqref{eq:gradient_F}
is symmetric.
\end{thm}
\begin{proof}
Assume that the SPD matrix $\boldsymbol{M}=\boldsymbol{U}\boldsymbol{\varLambda}\boldsymbol{U}^{\top}\in\mathbb{R}^{d\times d}$,
where $\boldsymbol{\varLambda}=\text{diag}\left(\lambda_{1},\lambda_{2}\cdots,\lambda_{d}\right)\in\mathbb{R}^{d\times d}$
and $\boldsymbol{U}=[\boldsymbol{U}_{1},\boldsymbol{U}_{2}\cdots,\boldsymbol{U}_{d}]\in\mathbb{R}^{d\times d}$
are the matrices consisting of distinct eigenvalues and unit eigenvectors
of $\boldsymbol{M}$, respectively. For $H_{1}(\boldsymbol{M}),H_{2}(\boldsymbol{M}),\cdots,H_{N}(\boldsymbol{M})$
in Eq.~\ref{eq:gradient_2}, we simply let $H(\boldsymbol{M})=\widehat{\boldsymbol{x}}^{\top}\boldsymbol{U}\boldsymbol{h}(\boldsymbol{\varLambda})\boldsymbol{U}^{\top}\widehat{\boldsymbol{x}}$.
By using the Maclaurin's formula \cite{russell1996principles_math}
on each eigenvalues, namely $\boldsymbol{h}(\boldsymbol{\varLambda})=\sum_{i=0}^{+\infty}\frac{h^{(i)}(0)}{i!}\boldsymbol{\varLambda}^{i}$,
then $H(\boldsymbol{M})$ equals to
\begin{equation}
\!\widehat{\boldsymbol{x}}^{\top}\boldsymbol{U}\sum_{i=0}^{+\infty}\nolimits\frac{h^{(i)}(0)}{i!}\boldsymbol{\varLambda}^{i}\boldsymbol{U}^{\top}\widehat{\boldsymbol{x}}=\sum_{i=0}^{+\infty}\nolimits\frac{h^{(i)}(0)}{i!}\widehat{\boldsymbol{x}}^{\top}\boldsymbol{M}^{i}\widehat{\boldsymbol{x}}.\!
\end{equation}
Since the gradient of $\widehat{\boldsymbol{x}}^{\top}\boldsymbol{M}^{i}\widehat{\boldsymbol{x}}$
is symmetric for any $i\in\mathbb{Z}^{+}$ \cite{petersen2008matrix_cookbook},
the summation of the gradient matrices is also symmetric and the proof
is completed.
\end{proof}
Now it has been proved that $\boldsymbol{M}-\rho\nabla D(\boldsymbol{M})$
remains symmetric during iterations, and thus the projection $\mathcal{P}_{\mathcal{S}}(\cdot)$
ensures the SPD property of $\boldsymbol{M}$, \emph{i.e.,}\ the
constraint $\boldsymbol{M}\in\mathcal{S}$ is always satisfied. It
means that the gradient descent is always performed in the feasible
region of the optimization problem. Then according to the theoretically
proved convergence of the gradient descent method \cite{boyd2004convex},
Algorithm \ref{alg:Solving-AML-Eq.} converges to the stationary point
of Eq.~\eqref{eq:optimization_model}.

\section{Experiments\label{sec:Experiments}}

In this section, empirical investigations are conducted to validate
the effectiveness of AML. In detail, we first visualize the mechanism
of the proposed AML on a synthetic dataset. Then we compare the performance
of the proposed method AML (Algorithm \ref{alg:Solving-AML-Eq.})
with three classical metric learning methods (ITML \cite{davis2007information},
LMNN \cite{weinberger2009distance} and FlatGeo \cite{meyer2011regression_metric})
and five state-of-the-art metric learning methods (RVML \cite{NIPS2015_RVML},
GMML \cite{zadeh2016geometric}, ERML \cite{yang2016_empirical},
DRML \cite{harandi2017joint_metric_learning}, and DRIFT \cite{ye_learning_M-distance})
on seven benchmark classification datasets. Next, all methods are
compared on three practical datasets related to face verification
and image matching. Finally, the parametric sensitivity of AML is
studied.

\subsection{Experiments on Synthetic Dataset}

We first demonstrate the effectiveness of AML on a synthetic dataset
which contains $200$ training examples and $200$ test examples across
two classes. The data points are sampled from a $10$-dimensional
normal distribution, and are visualized by the first two principal
components \cite{abdi2010_PCA}. As shown in Figs.~\ref{fig:Visualization of AML}(a)
and (b), the training set is clean, but the test examples belonging
to two classes overlap in the intersection region and lead to many
ambiguous test data pairs.

Since GMML \cite{zadeh2016geometric} shares the same loss function
with our AML, and the only difference between GMML and AML is that
AML utilizes the adversarial points while GMML does not, so here we
only compare the results of GMML and AML to highlight the usefulness
of our adversarial framework. The training and test results of both
methods are projected to Euclidean space by using the learned metrics.
It can be found that the traditional metric learning model GMML simply
learns the optimal metric for training data, and thus its corresponding
projection matrix directly maps the data points onto the horizontal-axis
in the training set (Fig.~\ref{fig:Visualization of AML}(c)). However,
such a learned metric is confused by the data points in the test set
(Fig.~\ref{fig:Visualization of AML}(d)) as the two classes are very
close to each other in the test set. As a result, the two classes
are not well-separated by the learned metric. In contrast, the proposed
AML not only produces very impressive result on the training set (Fig.~\ref{fig:Visualization of AML}(e)),
but also generates very discriminative results on test set (Fig.~\ref{fig:Visualization of AML}(f)).
The test data belonging to the same class is successfully grouped
together while the examples of different classes are separated apart.
This good performance of AML owes to the adversarial data pairs as
shown by ``$\times$'' in Fig.~\ref{fig:Visualization of AML}(a).
Such difficult yet critical training pairs effectively cover the ambiguous
situations that may appear in the test set, and therefore enhancing
the generalizability and discriminability of our AML. 
\begin{figure}
\begin{spacing}{0.3}
\noindent \begin{centering}
\includegraphics[scale=0.38]{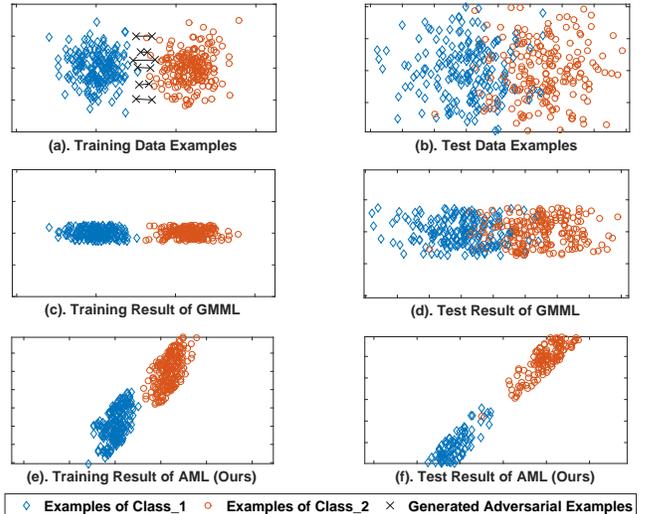}
\par\end{centering}
\end{spacing}
\begin{spacing}{0.5}
\noindent \caption{\label{fig:Visualization of AML}Visual comparison of GMML and the
proposed AML on synthetic dataset. Although the satisfactory training
result is obtained by the traditional metric learning model GMML,
it cannot well handle the test cases with ambiguous pairs. In contrast,
our proposed AML shows good discriminability on both training and
test sets. The reason lies in that the generated adversarial training
data pairs help to boost the discriminability of AML.}
\end{spacing}
\end{figure}

\subsection{Experiments on Classification \label{subsec:Classification-and-Verification}}

To evaluate the performances of various compared methods on classification
task, we follow existing works \cite{xie2013multi_metric_IJCAI,LinLiang_PAMI2017_Generalized_Metric}
and adopt the $k$-NN classifier ($k=5$) based on the learned metrics
to investigate the classification error rate. The datasets are from
the well-known UCI repository \cite{asuncion2007uci}, which include
\emph{Breast-Cancer}, \emph{Vehicle}, \emph{German-Credit},\emph{
Image-Segment}, \emph{Isolet}, \emph{Letters} and \emph{MNIST}. The
number of contained classes, examples and features are displayed in
Table \ref{tab:Classification-error-rates}. We compare all methods
over $20$ random trials. In each trial, $80\%$ of examples are randomly
selected as the training examples, and the rest are used for testing.
By following the recommendation in \cite{zadeh2016geometric}, the
training pairs are generated by randomly picking up $1000c(c-1)$
pairs among the training examples. The parameters in our method such
as $\alpha$ and $\beta$ are tuned by searching the grid $\{10^{-3},10^{-2},\cdots,10^{3}\}$.
The parameters for baseline algorithms are also carefully tuned to
achieve the optimal results. The average classification error rates
of compared methods are showed in Table \ref{tab:Classification-error-rates},
and we find that AML obtains the best results when compared with other
methods in most cases. 
\begin{table*}
\begin{singlespace}
\noindent \caption{\label{tab:Classification-error-rates}Classification error rates
of $k$-nearest neighbor classifier based on the metrics output by
different methods. Three numbers below each datasets correspond to
the feature dimensionality ($d$), number of classes ($c$) and number
of examples ($n$). The best two results in each dataset are highlighted
in {\color{red}red} and {\color{blue}blue}, respectively.}

\end{singlespace}
\begin{spacing}{0.7}
\noindent \centering{}%
\begin{tabular}{l|c|c|c|c|c|c|c|c}
\hline 
\multirow{2}{*}{{\scriptsize{}Methods}} & \textbf{\scriptsize{}Breast-Cancer} & \textbf{\scriptsize{}Vehicle} & \textbf{\scriptsize{}German-Credit} & \textbf{\scriptsize{}Image-Segment}{\scriptsize{} } & \textbf{\scriptsize{}Isolet} & \textbf{\scriptsize{}Letters} & \textbf{\scriptsize{}MNIST} & \multirow{2}{*}{{\scriptsize{}References}}\tabularnewline
 & {\scriptsize{}$9,2,699$} & {\scriptsize{}$18,4,848$} & {\scriptsize{}$24,2,1000$} & {\scriptsize{}$19,7,2310$} & {\scriptsize{}$617,26,7797$} & {\scriptsize{}$16,20,20000$} & {\scriptsize{}$784,10,4000$} & \tabularnewline
\hline 
{\scriptsize{}ITML} & {\scriptsize{}$.073\pm.010$} & {\scriptsize{}$.301\pm.051$} & {\scriptsize{}$.292\pm.032$} & {\scriptsize{}$.051\pm.022$} & {\scriptsize{}$.092\pm.011$} & {\scriptsize{}$.062\pm.002$} & {\scriptsize{}$.143\pm.023$} & {\scriptsize{}ICML 2007}\tabularnewline
{\scriptsize{}LMNN} & {\scriptsize{}$.052\pm.012$} & {\scriptsize{}$.234\pm.021$} & {\scriptsize{}$.274\pm.013$} & {\scriptsize{}$.027\pm.008$} & {\color{blue}{\scriptsize{}$.032\pm.012$}} & {\scriptsize{}$.042\pm.020$} & {\scriptsize{}$.174\pm.020$} & {\scriptsize{}JMLR 2009}\tabularnewline
{\scriptsize{}FlatGeo} & {\scriptsize{}$.075\pm.021$} & {\scriptsize{}$.283\pm.042$} & {\scriptsize{}$.322\pm.031$} & {\scriptsize{}$.042\pm.012$} & {\scriptsize{}$.065\pm.001$} & {\scriptsize{}$.082\pm.016$} & {\scriptsize{}$.146\pm.043$} & {\scriptsize{}JMLR 2011}\tabularnewline
{\scriptsize{}RVML} & {\scriptsize{}$.045\pm.005$} & {\color{blue}{\scriptsize{}$.223\pm.032$}} & {\scriptsize{}$.280\pm.020$} & {\scriptsize{}$.035\pm.006$} & {\scriptsize{}$.035\pm.015$} & {\scriptsize{}$.056\pm.010$} & {\scriptsize{}$.132\pm.008$} & {\scriptsize{}NIPS 2015}\tabularnewline
{\scriptsize{}GMML} & {\color{red}{\scriptsize{}$.040\pm.015$}} & {\scriptsize{}$.235\pm.035$} & {\color{blue}{\scriptsize{}$.273\pm.021$}} & {\scriptsize{}$.031\pm.005$} & {\scriptsize{}$.075\pm.012$} & {\scriptsize{}$.051\pm.001$} & \textbf{\scriptsize{}$.120\pm.005$} & {\scriptsize{}ICML 2016}\tabularnewline
{\scriptsize{}ERML} & {\scriptsize{}$.045\pm.002$} & {\scriptsize{}$.245\pm.040$} & {\scriptsize{}$.279\pm.012$} & {\scriptsize{}$.036\pm.003$} & {\scriptsize{}$.065\pm.001$} & {\scriptsize{}$.054\pm.010$} & {\scriptsize{}$.133\pm.016$} & {\scriptsize{}IJCAI 2016}\tabularnewline
{\scriptsize{}DRML} & {\scriptsize{}$.049\pm.011$} & {\scriptsize{}$.246\pm.053$} & {\scriptsize{}$.279\pm.034$} & {\color{blue}{\scriptsize{}$.026\pm.007$}} & {\scriptsize{}$.039\pm.014$} & {\color{blue}{\scriptsize{}$.041\pm.027$}} & {\color{blue}{\scriptsize{}$.118\pm.012$}} & {\scriptsize{}ICML 2017}\tabularnewline
{\scriptsize{}DRIFT} & {\scriptsize{}$.043\pm.005$} & {\scriptsize{}$.240\pm.023$} & {\scriptsize{}$.278\pm.026$} & {\scriptsize{}$.032\pm.007$} & {\scriptsize{}$.034\pm.033$} & {\scriptsize{}$.049\pm.012$} & {\scriptsize{}$.127\pm.034$} & {\scriptsize{}IJCAI 2017}\tabularnewline
{\scriptsize{}AML} & {\color{blue}{\scriptsize{}$.044\pm.001$}} & {\color{red}{\scriptsize{}$.203\pm.021$}} & {\color{red}{\scriptsize{}$.251\pm.012$}} & {\color{red}{\scriptsize{}$.024\pm.004$}} & {\color{red}{\scriptsize{}$.029\pm.009$}} & {\color{red}{\scriptsize{}$.032\pm.020$}} & {\color{red}{\scriptsize{}$.105\pm.001$}} & {\scriptsize{}Ours}\tabularnewline
\hline 
\end{tabular}
\end{spacing}
\end{table*}
\begin{figure*}
\begin{spacing}{0.3}
\noindent \begin{centering}
\includegraphics[scale=0.42]{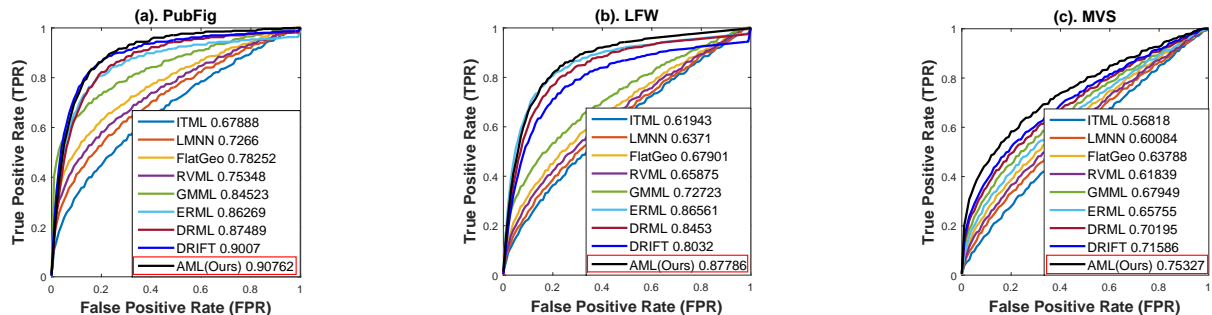}
\par\end{centering}
\end{spacing}
\begin{spacing}{0.3}
\caption{\label{fig:ROC-of-different}ROC curves of different methods on (a)
PubFig, (b) LFW and (c) MVS datasets. AUC values are presented in
the legends. }
\end{spacing}
\end{figure*}

\subsection{Experiments on Verification}

We also use two face datasets and one image matching dataset to evaluate
the capabilities of all compared methods on image verification task.
The \emph{PubFig }face dataset \cite{nair2010rectified_pubfig} consists
of of $2\times10^{4}$ pairs of images belonging to $140$ people,
in which the first $80\%$ pairs are selected for training and the
rest are used for testing. Similar experiments are performed on the
\emph{LFW} face dataset \cite{huang2007_LFW} which includes $13233$
unconstrained face images of $5749$ individuals. The image matching
dataset \emph{MVS }\cite{PAMI_2011_MVS_dataset}\emph{ }consists of
$64\times64$ gray-scale image sampled from 3D reconstructions of
the Statue of Liberty (LY), Notre Dame (ND) and Half Dome in Yosemite
(YO). By following the settings in \cite{simo2015discriminative},
LY and ND are put together to form a training set with over $10^{5}$
image patch pairs, and $10^{4}$ patch pairs in YO are used for testing.
The adopted features for above experiments are extracted by DSIFT
\cite{cheung2009_NSIFT} and Siamese-CNN \cite{zagoruyko2015_2-channel_siamase}
for face datasets (\emph{i.e.}\ PubFig and LFW) and image patch dataset
(\emph{i.e.}\ MVC), respectively. We plot the Receiving Operator
Characteristic (ROC) curve by changing the thresholds of different
distance metrics. Then the values of Area Under Curve (AUC) are calculated
to evaluate the performances quantitatively. From the ROC curves and
AUC values in Fig.~\ref{fig:ROC-of-different}, it is clear to see
that AML consistently outperforms other methods.

\subsection{Parametric Sensitivity}

In our proposed AML, there are two parameters which might influence
the model performance. Parameter $\alpha$ in Eq.~\eqref{eq:learning_loss}
determines the weights between original training data and generated
adversarial data, and parameter $\beta$ in Eq.~\eqref{eq:adversarial_min}
controls the size of neighborhood producting adversarial data.
\begin{figure}
\begin{spacing}{0.2}
\noindent \includegraphics[scale=0.4]{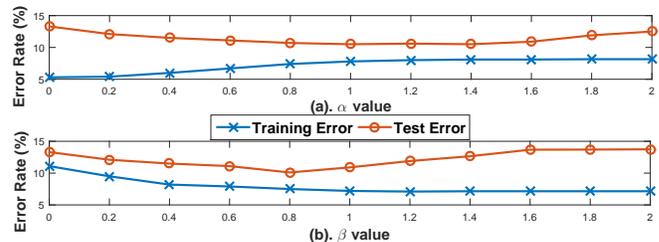}
\end{spacing}

\begin{spacing}{0.5}
\noindent \caption{\label{fig:Parametric Sensitivity}Parametric Sensitivity on MNIST
dataset. (a) Error rates under different $\alpha$ values $(\beta=1)$;
(b) Error rates under different $\beta$ values $(\alpha=1)$. }
\end{spacing}
\end{figure}

Intuitively, the growing of $\alpha$ increases the importance of
adversarial data, and the decrease of $\alpha$ makes the model put
more emphasize on the original training data. As shown in Fig. \ref{fig:Parametric Sensitivity}(a),
here we change the value of $\alpha$ and record the training error
and test error on the MNIST dataset that has been used in Section
\ref{subsec:Classification-and-Verification}. An interesting finding
is that, the training error grows when $\alpha$ increases in the
range $(0,1)$, but the test error consistently decreases at this
time. This is because tuning up $\alpha$ helps to alleviate the over-fitting
problem, and thus the test data with distribution bias from the training
data can be better dealt with. We also find that the training error
and test error make a compromise when $\alpha$ is around $1$, and
thus $1$ is an ideal choice for the parameter $\alpha$. Similarly,
the parameter $\beta$ varies within $(0,2)$ and the corresponding
training error and test error are recorded in Fig. \ref{fig:Parametric Sensitivity}(b).
It is clear to find that $\beta\approx0.8$ renders the highest test
accuracy and the performance is generally stable around $0.8$, which
mean that this parameter can be easily tuned for practical use. 

\section{Conclusion\label{sec:Conclusion-and-Further}}

In this paper, we propose a metric learning framework, named Adversarial
Metric Learning (AML), which contains two important competing stages
including confusion and distinguishment. The confusion stage adaptively
generates adversarial data pairs to enhance the capability of learned
metric to deal with the ambiguous test data pairs. To the best of
our knowledge, this is the first work to introduce the adversarial
framework to metric learning, and the visualization results demonstrate
that the generated adversarial data critically enriches the knowledge
for model training and thus making the learning algorithm acquire
the more reliable and precise metric than the state-of-the-art methods.
Furthermore, we show that such adversarial process can be compactly
unified into a bi-level optimization problem, which is theoretically
proved to have a globally convergent solver. Since the proposed AML
framework is general in nature, it is very promising to apply AML
to more deep neural networks based metric learning models for the
future work.

\small

\bibliographystyle{named}
\bibliography{ijcai18}

\end{document}